\documentclass[A4paper, 10pt, twocolumn]{article}

\usepackage{amsmath}
\usepackage{graphics} 			
\usepackage{epsfig} 			
\usepackage{mathptmx} 			
\usepackage{times} 				
\usepackage{amsthm} 			
\usepackage{amsfonts}
\usepackage{amstext}
\usepackage{amssymb}
\usepackage{stmaryrd}
\usepackage[normalem]{ulem}     
\usepackage{algpseudocode}
\usepackage{multirow}
\usepackage{subfigure}
\usepackage{color}

\newtheorem{theorem}{Theorem}
\newtheorem*{theorem*}{Theorem}
\newtheorem{corollary}{Corollary}
\newtheorem{definition}{Definition}

\newtheorem{proposition}{Proposition}
\newtheorem{remark}{Remark}

\def\real{\mbox{\rm I\kern-0.18em R}}

\title{\LARGE \bf Optimal Synthesis for Nonholonomic
Vehicles With Constrained Side Sensors}

\author{Paolo Salaris$^*$, Lucia
  Pallottino$^*$ and Antonio Bicchi$^*$ \thanks{This work was
    supported by E.C. contracts n.224428 CHAT, n.224053 CONET (Cooperating Objects Network of Excellence), n. 257462 HYCON2 (Network of Excellence) and n.2577649 Planet.}
  \thanks{$^*$ The Interdept. Research Center
    ``En\-ri\-co Pi\-ag\-gio'', U\-ni\-ver\-si\-ty of Pisa, via
    Diotisalvi 2, 56100 Pisa, Italy.  {\tt\footnotesize
      paolo.salaris,l.pallottino,bicchi@ing.unipi.it} } }

\begin{document}

\maketitle
\thispagestyle{empty}
\pagestyle{empty}

\begin{abstract}
We present a complete characterization of shortest paths to a goal position for a vehicle with unicycle kinematics and a limited range sensor, constantly keeping a given landmark in sight. Previous work on this subject studied the optimal paths in case of a frontal, symmetrically limited Field--Of--View (FOV). In this paper we provide a generalization to the case of arbitrary  FOVs, including the case that the direction of motion is not an axis of symmetry for the FOV, and even that it is not contained in the FOV. The provided solution is of particular relevance to applications using side-scanning, such as e.g. in underwater sonar-based surveying and navigation.

\end{abstract}


\section{Introduction}
In several mobile robot applications, a vehicle with nonholonomic kinematics of the unicycle type, equipped with a limited range sensor systems, has to reach a target while keeping some environment landmark in sight.
For example, in the Visual--Based control field the vehicle usually has an on-board monocular camera with limited Field--Of--View (FOV) and, subject to nonholonomic constraints on its motion, must move maintaining in sight one or more specified features of the environment. On the other hand, in the field of underwater surveying and navigation, a common task for Autonomous Underwater Vehicles (AUV) equipped with side sonar scanners is to detect and recognize objects (mines, wrecks or archeological find, etc.) on the sea bed (see e.g.~\cite{AUV,AUV2}). Side-scan sonar is a category of sonar systems that is used to efficiently create an image of large areas of the sea. 
Therefore, in order to recognize objects AUVs must move keeping them inside the limited range of the sensor.

Motivated by those application, in this paper we propose the study of optimal (shortest) paths for a nonholonomic vehicle moving in a plane to reach a target position while making so that a given landmark fixed in the plane is kept inside a planar cone moving with the robot.

The literature of optimal (shortest) paths stems mainly from the seminal work on unicycle vehicles with a bounded turning radius by Dubins~\cite{Dubins57}. Dubins has characterized the finite family of optimal paths for the particular vehicle while a complete optimal control synthesis for this problem has been reported in~\cite{bui94}. Later on, a similar problem with the car moving both forward and backward has been solved with different approaches in~\cite{ReedsShepp90}, \cite{SussmannTang91}. In particular, in~\cite{Soueres96} the optimal control synthesis for the Reeds$\&$Shepp car has been provided.
Minimum wheel rotation paths in for differential-drive robots have been considered in~\cite{LaValle09}. More recently, also the problem of determining minimum time trajectory has been taken into account in~\cite{Wang09}, \cite{Balkcom06} and \cite{Chyba01} for particular classes of robots, e.g. latter is on underwater robots. Finally, previous works on the same subject of this paper (\cite{SFPB-TRO09}, \cite{PSFB-CDC09}, \cite{HutchinsonBM07}) have studied the optimal paths in case of a vehicle with a limited on-board camera but only with a symmetric FOV with respect to the forward direction of the robot. In this paper, we present a more general synthesis of shortest paths in case of side sensor systems, like side sonar scanners on UAVs, where the forward direction is not necessarily included inside the sensor range modeled as a cone centered on the vehicle. The impracticability of paths that point straight to the feature lead to a more complex analysis of the reduction to a finite and sufficient family of optimal paths by excluding particular types of path.

In the rest of the paper, we provide a complete optimal synthesis for the problem, i.e., a finite language of optimal control words (at most 15 words, depending on orientation of the sensor with respect to the forward direction), and a global partition of the motion plane induced by shortest paths, such that a word in the optimal language is univocally associated to a region and completely describes the constrained shortest path from any starting point in that region to the goal point.

\section{Problem Definition}

Consider a vehicle moving on a plane where a right-handed reference
frame  $\langle W\rangle$ is defined with origin in $O_W$ and axes $X_w, Z_w$.
The configuration of the vehicle is described by $\xi(t)=(x(t),z(t),\theta(t))$, where $(x(t),z(t))$ is the position in
$\langle W\rangle$ of a reference point in the vehicle, and $\theta(t)$ is the vehicle heading with respect to the $X_w$ axis (see fig.~\ref{fig:VehicleCoordinates}). We assume that the dynamics of the vehicle are negligible, and that the forward and angular velocities, $\nu(t)$ and $\omega(t)$ respectively, are the control inputs to the kinematic model.
Choosing polar coordinates for the vehicle $\eta=[\rho\,\psi\,\beta]^T$ (see fig.~\ref{fig:VehicleCoordinates}),
the kinematic model of the unicycle-like robot is
\begin{equation}
\label{eq:PolarVehicle}
\begin{bmatrix}
\dot\rho \\
\dot\psi \\
\dot\beta
\end{bmatrix}=
\begin{bmatrix}
-\cos\beta & 0 \\
\frac{\sin\beta}{\rho} & 0 \\
\frac{\sin\beta}{\rho} & -1
\end{bmatrix}
\begin{bmatrix}
\nu \\
\omega
\end{bmatrix}.
\end{equation}
\begin{figure}[t]
\centering
\includegraphics[angle=90,width=0.7\columnwidth]{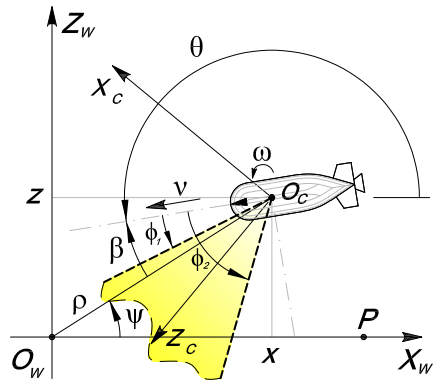}
\caption{Autonomous vehicle and systems coordinates. The vehicle's task is to
  reach $P$ while keeping $O_W$ within a limited sensor range modelled as a planar cone (highlighted in color).}
\label{fig:VehicleCoordinates}
\end{figure}
We consider vehicles with bounded velocities which can turn on the spot. In other words, we assume
\begin{equation}
\label{eq:AdmisControls}
(\nu, \omega) \in U,
\end{equation}
with $U$ a compact and convex subset of $\real^2$, containing the origin in its interior.

The vehicle is equipped with a rigidly fixed sensor system with a reference frame $\langle C\rangle =\{O_c, X_c, Y_c, Z_c\}$ such that the center $O_c$ corresponds to the robot's center $[x(t),z(t)]^T$ and the forward sensor axis $Z_c$ forms an angle $\Gamma$ w.r.t the robot's forward direction.
Moreover, let $\delta$ be the characteristic angle of the cone characterizing the limited Sensor Range (SR) and let us consider the most interesting problem in which $\delta\leq\pi/2$.
Without loss of generality, we will consider $0\leq \Gamma\leq\frac{\pi}{2}$, so that, when $\Gamma = 0$ the $Z_c$ axis is aligned with the robot's forward direction (i.e., the particular case solved in~\cite{SFPB-TRO09}), whereas, when $\Gamma = \frac{\pi}{2}$, is aligned with the axle direction.  Consider $\phi_1=\Gamma-\frac{\delta}{2}$ and $\phi_2=\Gamma+\frac{\delta}{2}$ the angles between the robot's forward direction and the right or left sensor's border w.r.t. $Z_c$ axis, respectively. The restriction on $0\leq\Gamma=\frac{\phi_1+\phi_2}{2}\leq\frac{\pi}{2}$ will be removed at the end of this paper, and an easy procedure to obtain the subdivision for any value of $\Gamma$ will be given.

Without loss of generality, we consider the position of the robot target point $P$ to lay on the $X_W$ axis, with coordinates $(\rho,\,\psi)=(\rho_P,\,0)$.  We also assume that the feature to be kept within the SR is placed on the axis through the origin $O_W$ and perpendicular to the plane of motion. We consider a planar SR with characteristic angle $\delta=|\phi_2-\phi_1|$, which generates the constraints
\begin{eqnarray}
\beta-\phi_1\geq0\,,  \label{eq:S_R} \\
\beta-\phi_2\leq0\,. \label{eq:S_L}
\end{eqnarray}

Note that we place no restrictions on the vertical dimension of the sensor. Therefore, the height of the feature on the motion plane, which corresponds to its $Y_c$ coordinate in the sensor frame $\langle C \rangle$, is irrelevant to our problem. Hence, for our purposes, it is necessary to know only the projection of the feature on the motion plane, i.e., $O_W$.

The goal of this paper is to determine, for any point $Q\in\real^2$ in the robot space, the shortest path from $Q$ to $P$ such that the feature is maintained in the SR. In other words, we want to minimize the length of the path covered by the center of the vehicle under the {\em feasibility constraints}~\eqref{eq:PolarVehicle}, \eqref{eq:AdmisControls}, \eqref{eq:S_R}, and \eqref{eq:S_L}.

From the theory of optimal control with state and control constraints (see~\cite{Bryson}) it is possible to show that, when constraints~\eqref{eq:S_R} and~\eqref{eq:S_L} are not active, extremals curves, i.e., curves that satisfy necessary conditions for optimality, are straight lines (denoted by symbol $S$) and rotation on the spot (denoted by symbol $*$). On the other hand, when constraints~\eqref{eq:S_R} and~\eqref{eq:S_L} are active, the corresponding extremal maneuvers are two logarithmic spirals with characteristic angles $\phi_1$ and $\phi_2$ denoted by $T_1$ and $T_2$, respectively (see~\cite{SFPB-TRO09} for details).

Logarithmic spiral $T$ with characteristic angle $\phi>0$ ($\phi<0$) rotates counterclockwise (clockwise) around the feature. We refer to counterclockwise and clockwise spirals as \emph{Left} and \emph{Right}, and by symbols $T^L$ and $T^R$, respectively.  The adjectives ``left'' and ``right'' indicate the half--plane where the spiral starts for an on--board observer aiming at the feature.

Notice that, for $\phi_2=\pi/2$ the left sensor border is aligned with the axle direction and the spiral $T_2$ becomes a circle  centered in $O_W$ (denoted by $C$), whereas for $\phi_1=0$ the right sensor border is aligned with the direction of motion and $T_1$ becomes an half line through $O_W$ (denoted by $H$).
%

Extremal arcs can be executed by the vehicle in either forward or backward direction: we will hence use superscripts $+$ and $-$ to make this explicit (e.g., $S^-$ stands for a straight line executed backward).

We will build extremal paths consisting of sequences of symbols, or {\em words}, in the alphabet $\mathcal{A}=\{*,\,S^+,\,S^-,\,E_1^+,\,E_1^-,\,E_2^+,\,E_2^-\}$, where the actual meaning of symbols depends on angles $\Gamma$ and $\delta$ as in fig.~\ref{fig:AllCases}. Rotations on the spot ($*$) have zero length, but may be used to properly connect other maneuvers.


\begin{figure}[t!]
\centering
\subfigure[Frontal: $0 \leq \Gamma < \frac{\delta}{2}$, $E_1=T_1^L,\,E_2=T_2^R$.]{\label{fig:AF}\includegraphics[width=0.2\textwidth]{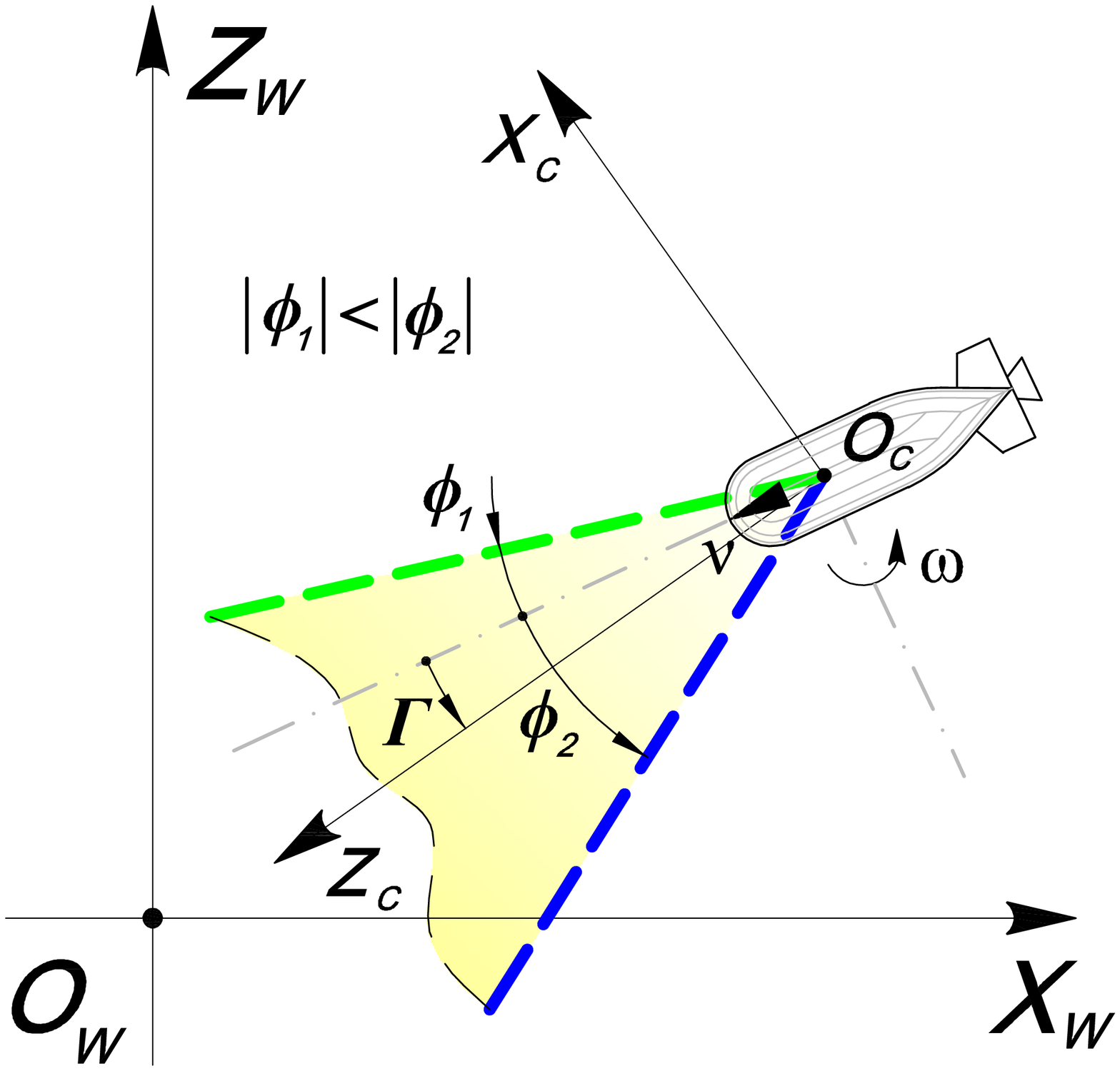}}
\quad
\subfigure[Borderline Frontal: $\Gamma=\frac{\delta}{2}$, $E_1=H,\,E_2=T_2^R$.]{\label{fig:BF}\includegraphics[width=0.2\textwidth]{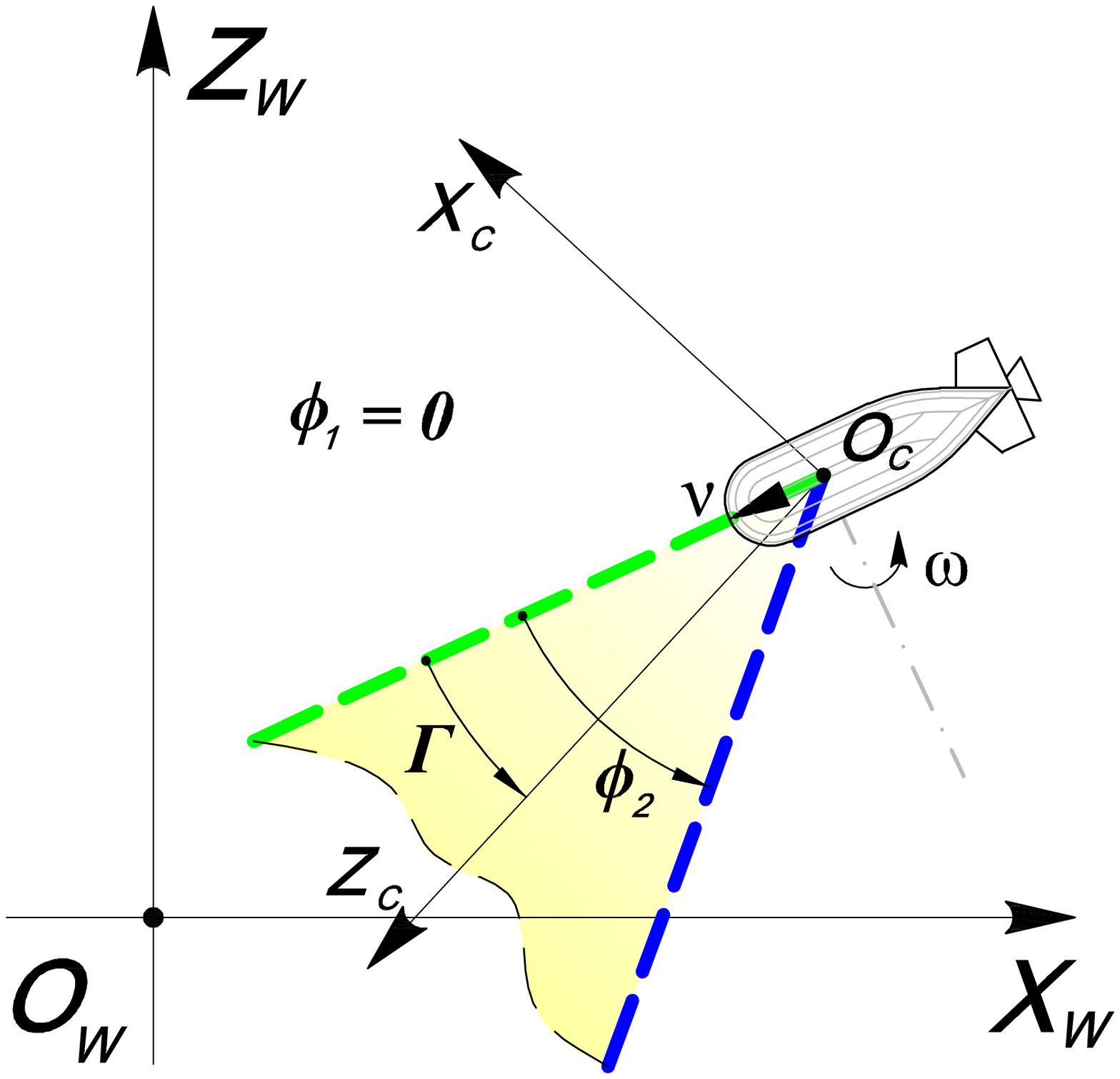}}
\quad
\subfigure[Side: $\frac{\delta}{2}<\Gamma<\frac{\pi-\delta}{2}$, $E_1=T_1^R,\,E_2=T_2^R$.]{\label{fig:Side}\includegraphics[width=0.2\textwidth]{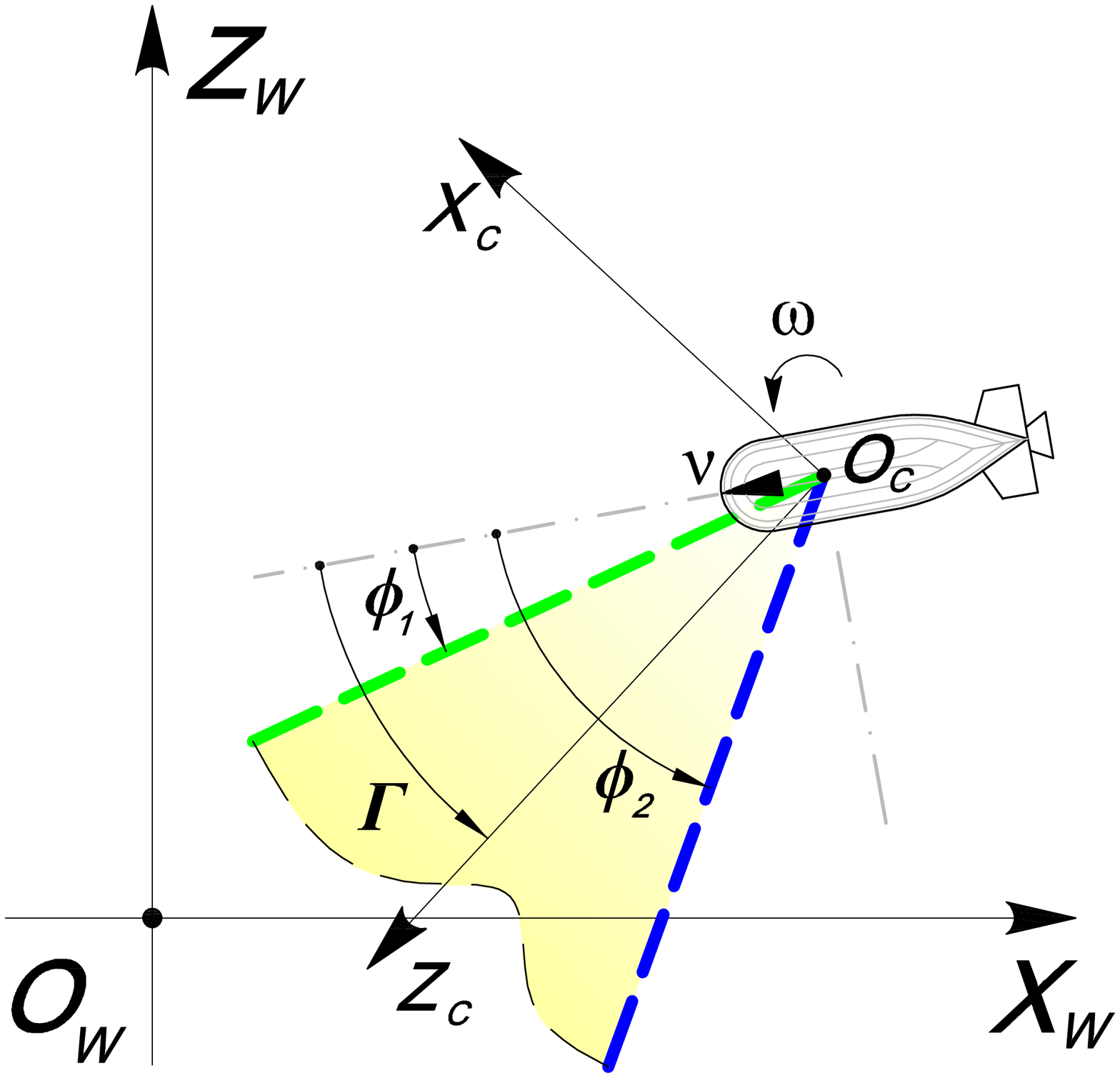}}
\quad
\subfigure[Borderline Side: $\Gamma=\frac{\pi-\delta}{2}$, $E_1=T_1^R,\,E_2=C$.]{\label{fig:BSide}\includegraphics[width=0.2\textwidth]{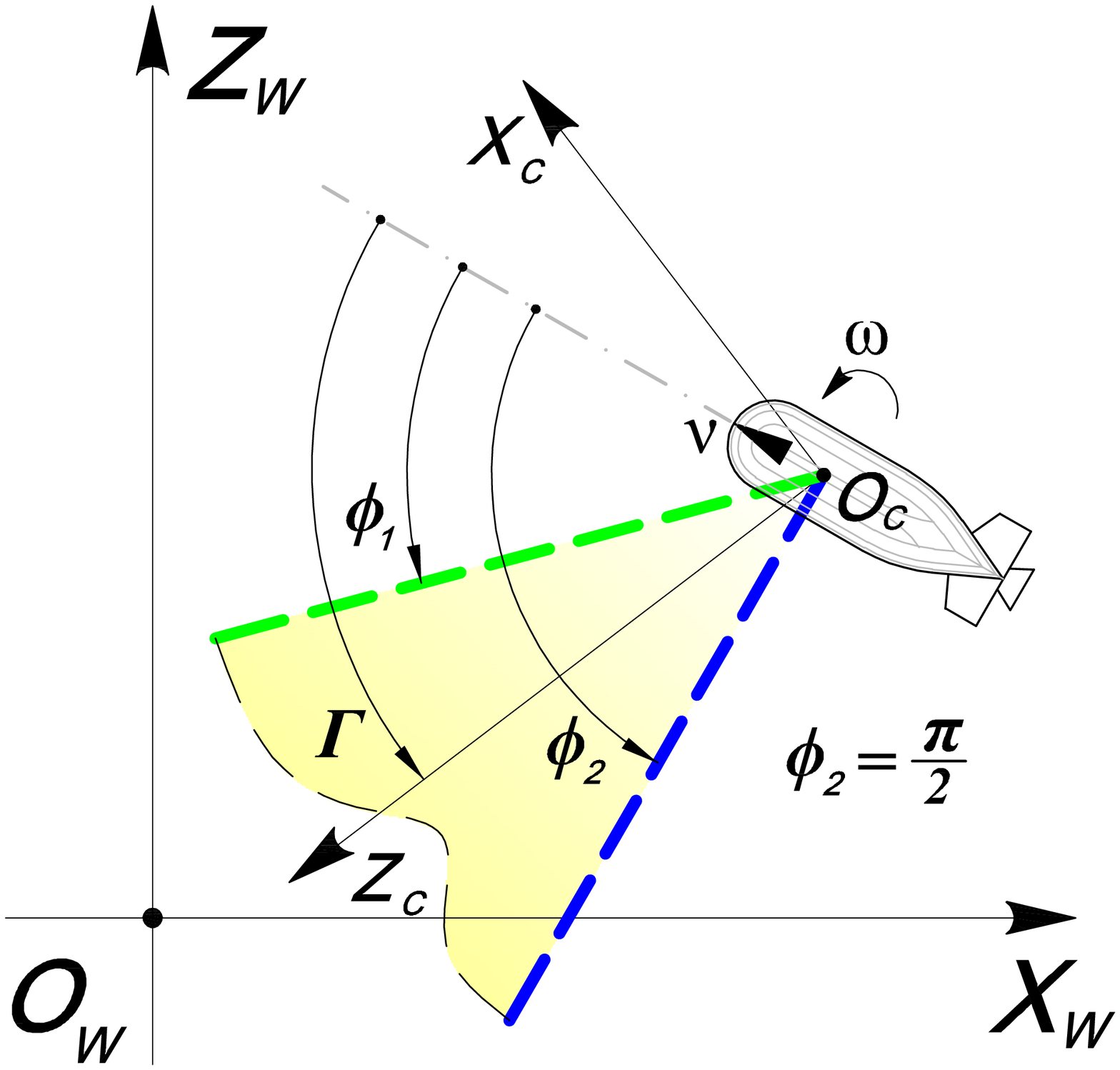}}
\quad
\subfigure[Lateral: $\frac{\pi-\delta}{2}<\Gamma<\frac{\pi}{2}$, $E_1=T_1^R,\,E_2=T_2^L$.]{\label{fig:AL}\includegraphics[width=0.2\textwidth]{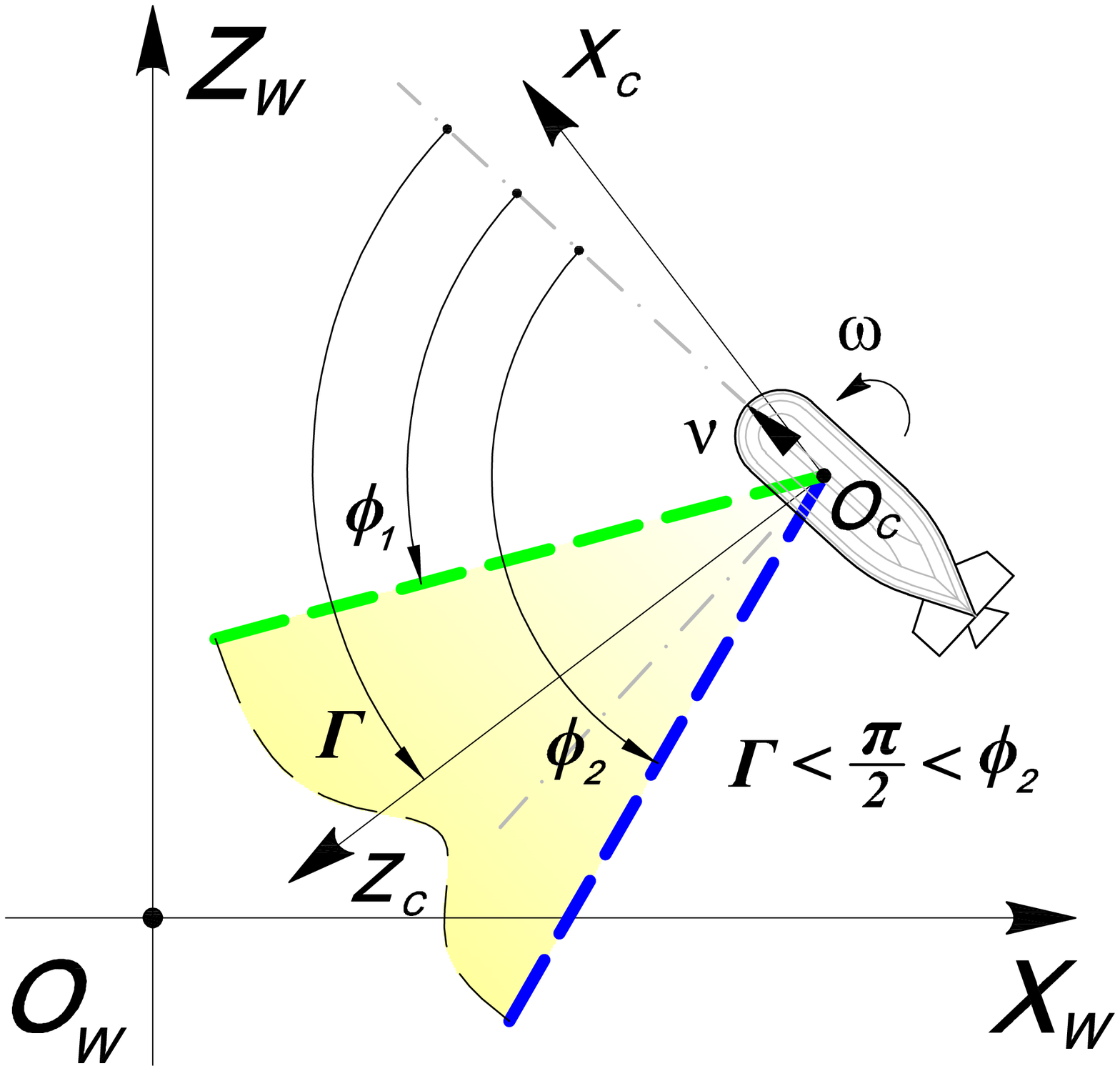}}
\quad
\subfigure[Symmetric Lateral: $\Gamma = \frac{\pi}{2}$, $E_1=T_1^R,\,E_2=T_2^L$.]{\label{fig:SL}\includegraphics[width=0.2\textwidth]{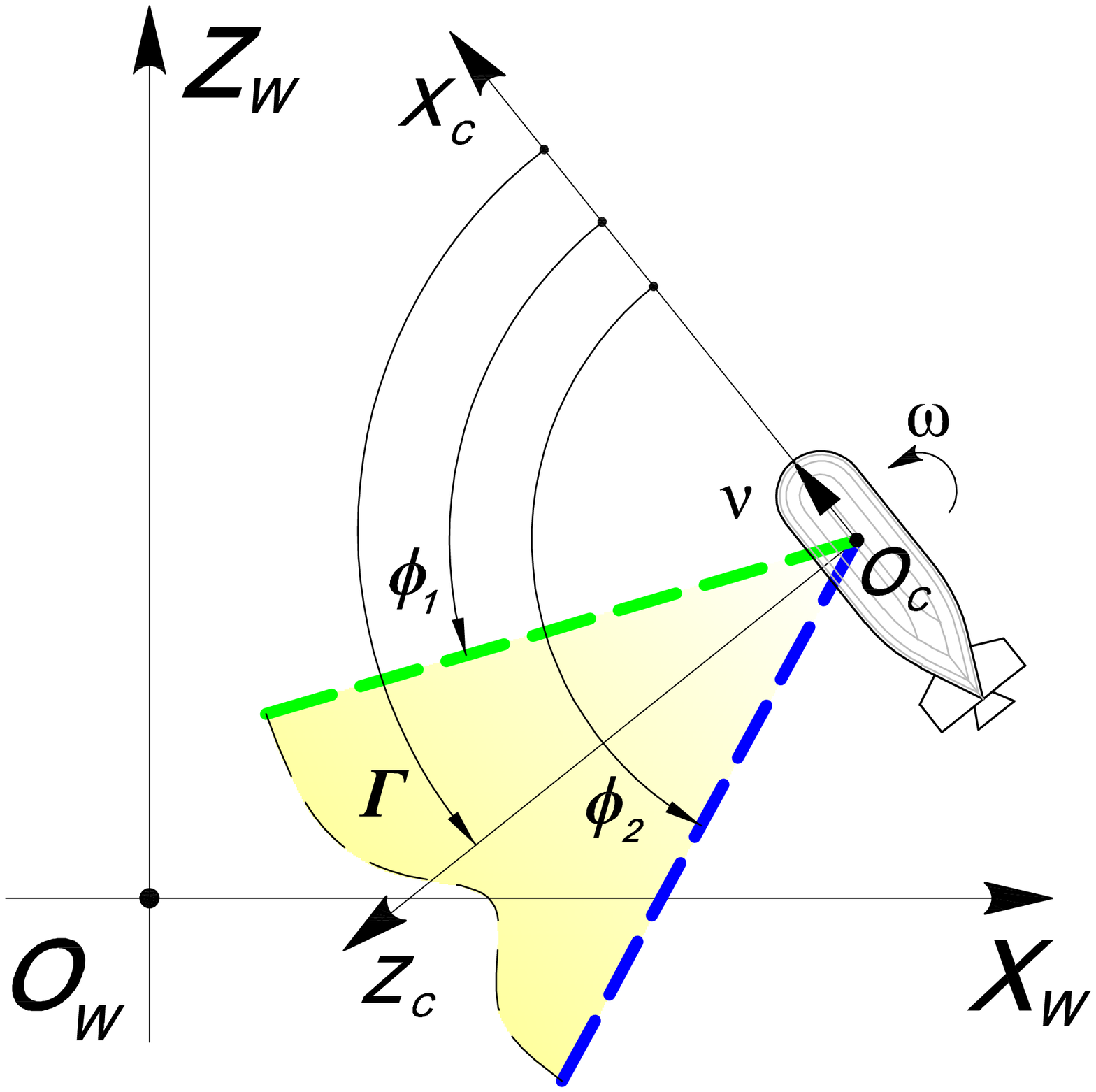}}
\caption{Sensor configuration depending on angles $\Gamma$ and $\delta$.}
\label{fig:AllCases}
\end{figure}


Let $\mathcal{L}_{\Gamma}$ be the set of possible words generated by the aforementioned symbols in $\mathcal{A}$ for each value of $\Gamma$. The rest of the paper is dedicated to showing that, due to the physical and geometrical constraints of the considered problem, a sufficient optimal finite language $\mathcal{L}_O\subset\mathcal{L}_{\Gamma}$ can be built such that, for any initial condition, it contains a word describing a path to the goal which is no longer than any other feasible path. Correspondingly, a partition of the plane in a finite number of regions is described, for which the shortest path is one of the words in $\mathcal{L}_O$.

\section{Shortest path synthesis}
\label{sec:OptimalPaths}

In this section, we introduce the basic tools that will allow us to study the optimal synthesis of the whole state space of the robot, beginning from points
on a particular sub--set of $\real^2$ such that the optimal paths are in a sufficient optimal finite language.


\begin{definition}
Given the target point $P = (\rho_P,\,0)$ in polar coordinates, and $Q\in\real^2\setminus O_W$, $Q = (\rho_Q,\psi_Q)$ with $\rho_Q\neq 0$, let
$f_Q:\real^2\rightarrow\real^2$ denotes the map
\begin{equation}\label{eq:f}
f_Q\left(\rho_G,\psi_G \right)= \left\{
\begin{aligned}
&\left(\frac{\rho_G\rho_P}{\rho_Q},\psi_G-\psi_Q\right) & \hbox{for } \rho_G\neq 0\\
&\left(0,0\right) & \hbox{otherwise.} \\
\end{aligned} \right.
\end{equation}
\end{definition}

The map $f_Q$ is the combination of a clockwise rotation by angle $\psi_G-\psi_Q$, and a scaling by a factor $\rho_P/\rho_Q$ that maps $Q$ in $P$.

\begin{remark}
The alphabet $\mathcal{A}$ is invariant w.r.t. rotation and scaling. However, it is not invariant w.r.t. axial symmetry, as it happened in the particular case (i.e., the Frontal case with $\Gamma=0$) considered in~\cite{SFPB-TRO09}, where the map $f_Q$ was defined as a combination of rotation, scaling and axial symmetry. For example, logarithmic spirals are self-similar and self-congruent (under scaling and rotation they are mapped into themselves). On the other hand, left (right) spirals are mapped into right (left) spirals through an axial symmetry and alphabet invariancy can be lost. Indeed, for example, considering the Side case alphabet (see fig.~\ref{fig:Side})  $\mathcal{A}_{Side}=\{*,\,S^+,\,S^-,\,T_1^{R+},\,T_1^{R-},\,T_2^{R+},\,T_2^{R-}\}$, and applying an axial symmetry we have $T_1^R\rightarrow T_1^L\notin\mathcal{A}_{Side}$, the same occurs for the Frontal alphabet with $\Gamma>0$.
\end{remark}

Let $\gamma$ be a path parameterized by $t\in[0,1]$ in the plane of motion
$\gamma(t) = (\rho(t),\,\psi(t))$.  Denote with ${\mathcal P}_Q$
the set of all feasible extremal  paths from $\gamma(0)=Q$ to $\gamma(1)=P$.
\begin{definition}
Given the target point $P=(\rho_P,\,0)$ and $Q=(\rho_Q,\psi_Q)$ with $\rho_Q\neq 0$, let the {\emph {path transform}} function $F_Q$
be defined as
\begin{equation}\label{eq:F}
\begin{aligned}
F_Q:&\,\,{\mathcal P}_Q\rightarrow {\mathcal P}_{f_Q(P)}\\
&\,\,\gamma(t)\mapsto f_Q(\gamma(1-t)),\;\forall t\in I.
\end{aligned}
\end{equation}
\end{definition}
Notice that $\tilde\gamma(t)=F_Q \left( \gamma(1-t) \right)$ corresponds to
$\gamma(t)$ transformed by $f_Q$ and followed in opposite direction. Indeed,
$\tilde\gamma$ is a path from $\tilde \gamma(0)=f_Q(P)=\left(\frac{\rho_P^2}{\rho_Q},\,-\psi_Q\right)$ to
$\tilde\gamma(1)=f_Q(Q)\equiv P$.

The $F$ map has some properties that make it very useful to the study of our problem in a way which is to some extent similar to what described (for a different $F$ map) in~\cite{SFPB-TRO09}. In particular, the locus of points $Q$ such that $f_Q(P) = Q$, is the circle  with center in $O_W$ and radius $\rho_P$. We will denote this circle  by $C(P)$ and the closed disk within $C(P)$ by $D(P)$.

$C(P)$ has an important role in the proposed approach since properties of $F_Q$ will allow us to solve the synthesis problem from points on $C(P)$, and hence to extend the synthesis to $D(P)$ and to the whole motion plane. Indeed, $\forall Q \in C(P) $ and $\forall \gamma\in{\mathcal P}_Q$, $F_Q(\gamma)\in {\mathcal P}_{f_Q(P)}$ with $f_Q(P)\in C(P)$, i.e., a path from a point on $C(P)$ to $P$ is mapped in a path from $C(P)$ to $P$.

Furthermore, $F_Q$ transforms an extremal in ${\mathcal A}$ in itself but followed in opposite direction. Hence, $F_Q$ maps extremal paths in $\mathcal{L}_\Gamma$ in extremal paths in $\mathcal{L}_\Gamma$. For example, let $w=S^-*H^-*S^+*T_2^{R+}$ be the word that characterize a path from $Q$ to $P$, the transformed path is of type  $z=T_2^{R-}*S^-*H^{+}*S^+$. With a slight abuse of notation, we will write $z=F_Q(w)$.

\begin{proposition}\label{prop:length}
Given $Q\in\real^2$ and a path $\gamma\in{\mathcal P}_Q$ of length $l$, the
length of the transformed path $\tilde\gamma=F_Q(\gamma)$ is
$\tilde l=\frac{\rho_P}{\rho_Q} l$.
\end{proposition}
The proof is easily obtained from a similar result in~\cite{SFPB-TRO09}.

Based on the properties of $F_Q$, optimal paths from points on $C(P)$ completely evolve inside $C(P)$. To prove this statement we first report the following result,
\begin{figure}[t]
\centering
\includegraphics[angle=90,width=0.85\columnwidth]{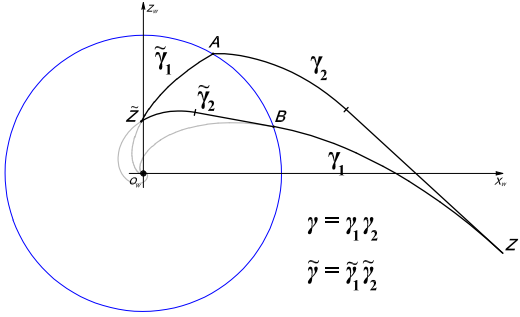}
\caption{An example for theorem~\ref{theo:InNOTOut}: path $\gamma=\gamma_1\gamma_2$ ($\gamma_1$ followed by $\gamma_2$) of type $T_2^{R-}S^-*T_1^{R+}$ from $A$ to $B$ is shortened by a path $\tilde\gamma=\tilde\gamma_1\tilde\gamma_2$ of type $T_1^{R+}*T_1^{R+}S^-$ by applying path transformation $F_{Z}$ to path $\gamma$.}
\label{fig:teorema1}
\end{figure}
\begin{theorem}
\label{theo:InNOTOut}
Given two points $A=(\rho_A,\,\psi_A)$ and $B=(\rho_B,\,\psi_B)$, with $\psi_A>\psi_B$ and $\rho=\rho_A=\rho_B$, and an extremal path $\gamma$ from $A$ to $B$ such that for each point $G$ of $\gamma$, $\rho_G>\rho$, there exists an extremal path $\tilde\gamma$ from $A$ to $B$ such that for each point $\tilde G$ of $\tilde\gamma$, $\rho_{\tilde G}<\rho$ and $\ell(\tilde\gamma)<\ell(\gamma)$ (see fig.~\ref{fig:teorema1}).
\end{theorem}

The proof of this theorem can be found in section~\ref{proof:Th1} in the Appendix.

\begin{figure*}[th!]
\centering
\includegraphics[width=1.2\columnwidth]{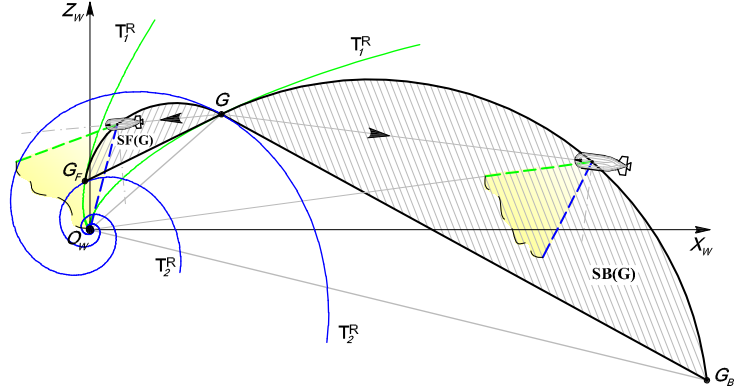}
\caption{Forward and backward straight path regions from $G$ for $\frac{\delta}{2}<\Gamma\leq \frac{\pi-\delta}{2}$.}
\label{fig:SecondCaseEye}
\end{figure*}
An important but straightforward consequence of the theorem is the following
\begin{corollary}
\label{coroll:InNOTOut}
For any path in $\mathcal{P}_Q$ with $Q \in C(P)$ there exists a shorter or equal-length path in $\mathcal{P}_Q$ that completely evolves in $D(P)$.
\end{corollary}

\section{Optimal paths for points on $C(P)$}
\label{sec:OptimalCS}
Our study of the optimal synthesis begins in this section addressing optimal paths from points on $C(P)$.  We first need to establish an existence result of optimal paths.

\begin{proposition}
\label{prop:existence}
For any $Q \in C(P)$ there exists a feasible shortest path to $P$.
\end{proposition}
\begin{proof}
Because of state constraints~\eqref{eq:S_R}, and~\eqref{eq:S_L}, and the restriction of optimal paths in $D(P)$ (Corollary~\ref{coroll:InNOTOut}) the state set is compact. Furthermore, it is possible to give an upper-bound on the optimal path length for all $\Gamma\in[0,\,\frac{\pi}{2}]$. Indeed, given a point $Q$ at distance $\rho$ from $O_W$ the optimal path to $P$ is shorter or equal to the following paths based on the value of $\Gamma$ and $\delta$:
\begin{itemize}
\item Frontal ($0\leq \Gamma\leq\frac{\delta}{2}$): $S^{+}*S^{-}$ or $H^{+}*H^{-}$
    of length $\rho+\rho_P$;
\item Side ($\frac{\delta}{2}<\Gamma<\frac{\pi-\delta}{2}$): $T^{R+}_{1Q}*T^{R-}_{2P}$, of length
    $\left(\frac{\rho-\rho_N}{\cos\phi_1} + \frac{\rho_P-\rho_N}{\cos\phi_2}\right)$,
    where $N$ is the intersection point between spirals $T^{R}_{1Q}$ and $T^{R}_{2P}$ through $Q$ and $P$ respectively;
\item Borderline Side ($\Gamma = \frac{\pi-\delta}{2}$: $T^{R+}_1*C^-_{P}$) of length
    $\left(\frac{\rho-\rho_P}{\cos\phi_1} + (\psi_N-\psi_P)\rho_P\right)$, where $N$ is the intersection point between spirals $T^{R}_1$ and $C_{P}$;
\item Lateral ($\frac{\pi-\delta}{2}<\Gamma\leq\frac{\pi}{2}$): $T^{L-}_{2Q}*T^{R-}_{1P}$, of length
    $\left(\frac{\rho-\rho_N}{\cos\phi_2} + \frac{\rho_P-\rho_N}{\cos\phi_1}\right)$,
    where $N$ is the intersection point between spirals $T^{L}_{2Q}$ and $T^{R}_{1P}$.
\end{itemize}
The system is also controllable because there always exists an intersection
point between two spirals (even if degenerated in half--lines or circumferences) with different characteristic angle even if both clockwise or counterclockwise around the feature. Hence, Filippov existence theorem for Lagrange problems can be invoked~\cite{Cesari}.
\end{proof}

In the following we provide a set of propositions that completely describe a sufficient optimal finite language for all values of $\Gamma\in[0,\,\frac{\pi}{2}]$.



\begin{definition}
\label{def:Eye1}
For any starting point $G = (\rho_G,\,\psi_G)$, let $SF(G)$ ($SB(G)$) be the set of all points reachable from $G$ with a forward (backward) straight line without violating the SR constraints.
\end{definition}


We denote with  $\partial SF_1(G)$ and $\partial SF_2(G)$ ($\partial SB_1(G)$ and $\partial SB_2(G)$) the borders of  $SF(G)$ ($SB(G)$). Also, let $C_i(G)$ denote the circular arcs from $G$ to $O_W$ such that, $\forall V \in C_i(G)$, $\widehat{G V O_W} = \pi-|\phi_i|$.

\begin{remark}
\label{rem:FirstCaseEye}
Based on simple geometric considerations, for any starting point $G = (\rho_G,\,\psi_G)$, for $0\leq \Gamma\leq\frac{\delta}{2}$ (Frontal Case), $SF(G)$ is the region between $\partial SF_2(G)=C_2(G)$ and $\partial SF_1(G)=C_1(G)$. Let $r_1(G)$ ($r_2(G)$) denote the half--line from $G$ forming an angle $\psi_G-\phi_1$ ($\psi_G-\phi_2$) with the $X_W$ axis (cf. fig.~\ref{fig:FirstCaseEye}). $SB(G)$ is the cone delimited by $\partial SB_1(G)=r_1(G)$ and $\partial SB_2(G)=r_2(G)$, outside circle with center in $O_W$ and radius $\rho_G$. Notice that, $SF(G)$ lays completely in the circle with center in $O_W$ and radius $\rho_G$. Moreover, in the particular case in which $\Gamma = \frac{\delta}{2}$ (Borderline Frontal Case), $E_1 = H$ and $\partial SF_1(G)$ degenerates in the chord ($\overline{GO_W}$) between $G$ and $O_W$, aligned with $r_1(G)$.
\end{remark}

As a consequence of Remark~\ref{rem:FirstCaseEye}, both $SF(G)$ and $SB(G)$ are tangent in $G$ to $T_1^L$ or $H$ and $T_2^R$.


\begin{figure}[t!]
\centering
\includegraphics[width=0.95\columnwidth]{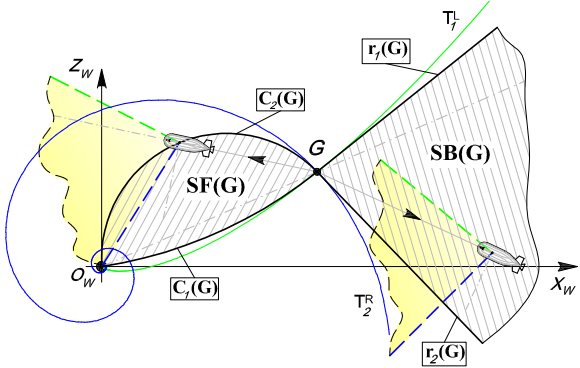}
\caption{Forward and backward straight path Regions from $G$ for $0\leq \Gamma\leq\frac{\delta}{2}$.}
\label{fig:FirstCaseEye}
\end{figure}

\begin{remark}
\label{rem:SecondCaseEye}
For any starting point $G = (\rho_G,\,\psi_G)$, and for $\frac{\delta}{2}<\Gamma\leq \frac{\pi-\delta}{2}$ (Side case), let $S_{G_F}$ be the chord between $G$ and $G_F = (\rho_G\frac{\sin\phi_1}{\sin\phi_2},\,\psi_G+(\phi_2-\phi_1))\in C_2(G)$,
i.e. such that $\widehat{O_WGG_F}=\phi_1$ (cf. fig.~\ref{fig:SecondCaseEye}). Naming with $C_{G_F}$ the arc between $G$ and $G_F$, $SF(G)$ is the region between arc $\partial SF_2(G)=C_{G_F}$ and chord $\partial SF_1(G)=S_{G_F}$.
Consider the rotation and scale that maps $G_F$ in $G$ and $G$ in $G_B$: we have $\partial SB_1(G)=\partial SF_1(G_B)$, i.e. $\partial SB_2(G)=\partial SF_2(G_B)$. Moreover, for all point $V$ on the circular arc $C_{G_B}$ from $G_B$ to $G$, angle $\widehat{G_BVO_W}=\pi-|\phi_2|$, and angle $\widehat{O_WG_BG}=\phi_1$. Notice that, in this case, $SF(G)$ lays completely in the circle with center in $O_W$ and radius $\rho_G$. Notice that, in the particular case in which $\Gamma = \frac{\pi-\delta}{2}$ (Borderline Side Case), $E_2 = C$ and $\partial SF_2(G)$ is an arc from $G$ to $G_F$ on a semicircle with diameter $\rho_G$.
\end{remark}



As a consequence of Remark~\ref{rem:SecondCaseEye}, $SF(G)$ is tangent in $G$ to $T_1^R$ and $T_2^R$ or $C$. 
Moreover, $SF(G)$ is tangent in $G_F$ to $T_1^R$ and $T_2^R$ or $C$, see fig.~\ref{fig:SecondCaseEye}.

Fig.~\ref{fig:LateralCaseEye} shows the $SF(G)$ and $SB(G)$ regions described in~\ref{rem:SecondCaseEye} for the Lateral case. Notice that, in this case, $SF(G)$ does not lay completely in the circle with center in $O_W$ and radius $\rho_G$.

\begin{figure}[t!]
\centering
\includegraphics[width=0.8\columnwidth]{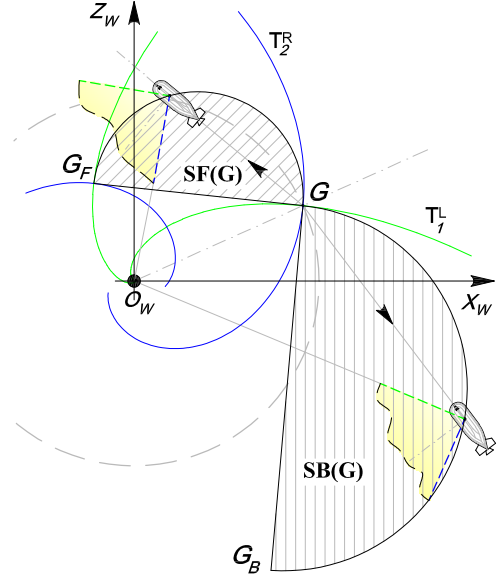}
\caption{Forward and backward straight path Regions from $G$ for $\frac{\pi-\delta}{2}\leq\Gamma\leq\frac{\pi}{2}$.}
\label{fig:LateralCaseEye}
\end{figure}


\begin{remark}
\label{rem:BorderEye}
Optimal forward (backward) straight arcs from any $G$ ends on $C_{G_F}$ ($C_{G_B}$) (see also~\cite{SFPB-TRO09} for details).
\end{remark}

Based on all the above properties, we are now able to obtain a sufficient family of optimal paths by excluding particular sequences of extremals.

\begin{theorem}\label{theo:rho_nomax}
Any path consisting in a sequence of a backward extremal arc followed by a forward extremal arc is not optimal.
\end{theorem}
The proof of this theorem, whose details can be found in section~\ref{prof:Th2} of the Appendix, is based on the fact that for continuity of paths, for any sequence of a backward extremal followed by a forward one, there exist points $A$ and $B$ that verify hypothesis of Theorem~\ref{theo:InNOTOut}.

\begin{theorem}\label{theo:EE}
Any path consisting in a sequence of an extremal arc $E_i$ and an extremal arc $E_j$ followed in the same direction is not optimal for any $i,j\in\{1,\,2\}$ with $i\neq j$.
\end{theorem}

Notice that the feasible sequences consisting of two extremals that we still need to discuss are those starting or ending with $S$ followed in any direction ($E^+E^-$ and $E^-E^+$ are obviously not optimal).

\begin{proposition}
\label{prop:FirstCaseS+}
From any starting point $A$, any path $\gamma$ of type $S^+*E^+_2$ and $S^+*E^-_1$ to $B$ can be shortened by a path of type $S^+E^+_2$ or $E^+_2*E^-_1$. Moreover, any path $\gamma$ of type $S^+*E^+_1$ or $S^+*E^-_2$ can be shortened by a path of type $E^+_1S^+$ or $E^+_1*E^-_2$.
\end{proposition}

\begin{figure*}[t!]
\centering
\subfigure[From $A$ to $B$, path $S^+*T^{R+}_2$ through $\text{z}$ and $\text{v}$ can be shortened by $S^+T^{R+}_2$ through $\text{v}$, where $S$ arc is tangent to $T^{R}_2$.]{\label{fig:P31}\includegraphics[width=0.35\textwidth]{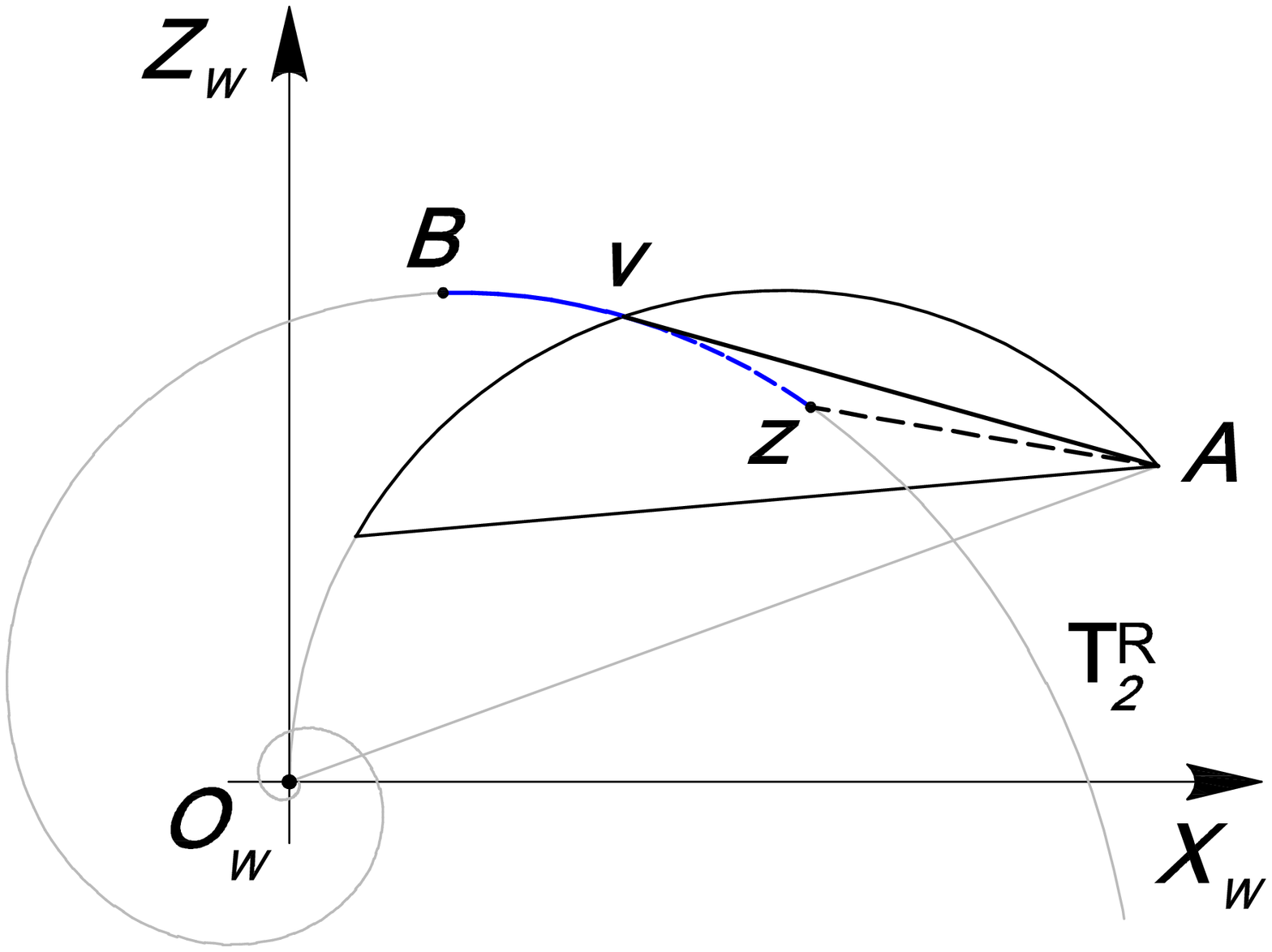}}
\;
\subfigure[From $A$ to $B^\prime$, path $S^+*T^{R-}_1$ through $\text{z}$ can be shortened by a path of type $S^+T^{R+}_2$ through $\text{v}$, whereas from $A$ to $B^{\prime\prime}$ by a path of type $T^{R+}_2*T^{R-}_1$ through $\text{g}$.]{\label{fig:P32}\includegraphics[width=0.37\textwidth]{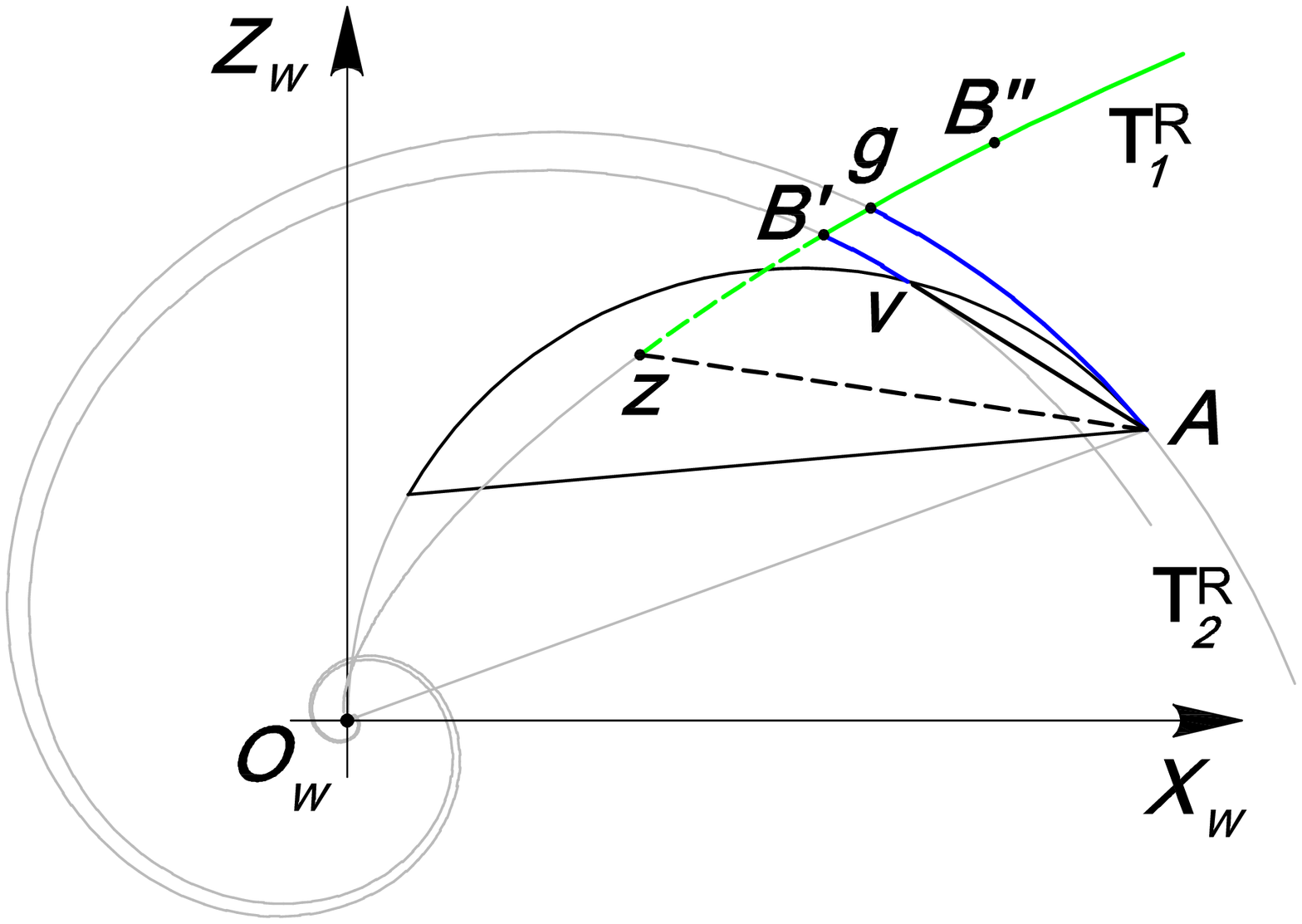}}
\;
\subfigure[From $A$ to $B$, path $S^+*T^{R+}_1$ through $\text{z}$ can be shortened by a path of type $T_1^{R+}S^+$ through $\text{v}$, where $S$ arc is tangent to $T_1^{R}$.]{\label{fig:P33}\includegraphics[width=0.35\textwidth]{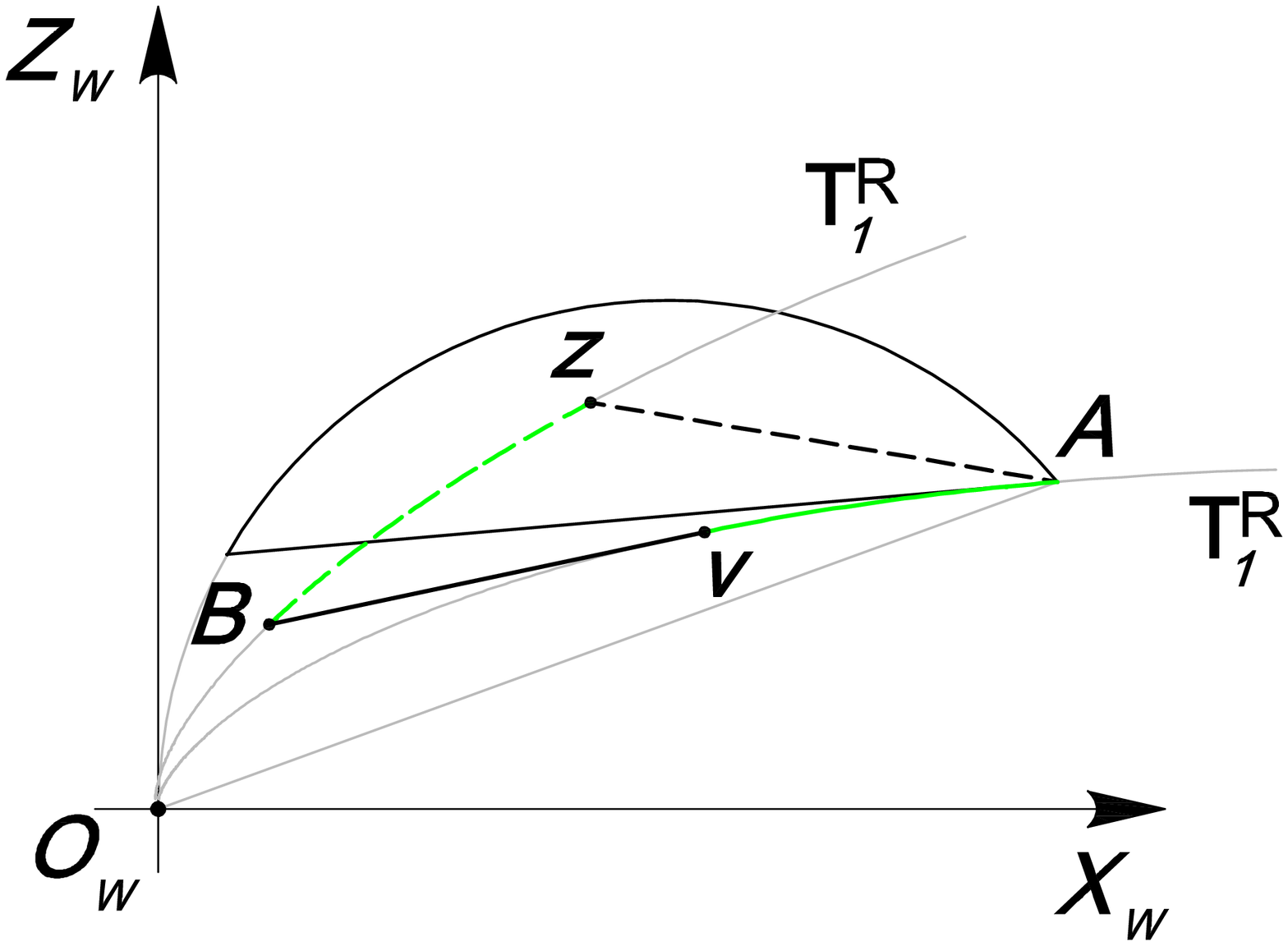}}
\;
\subfigure[From $A$ to $B^\prime$, path $S^+*T^{R-}_2$ through $\text{z}$ can be shortened by a path of type $T^{R+}_1S^+$ through $v$, whereas from $A$ to $B^{\prime\prime}$ by a path of type $T^{R+}_1*T^{R-}_2$ through $\text{g}$.]{\label{fig:P34}\includegraphics[width=0.4\textwidth]{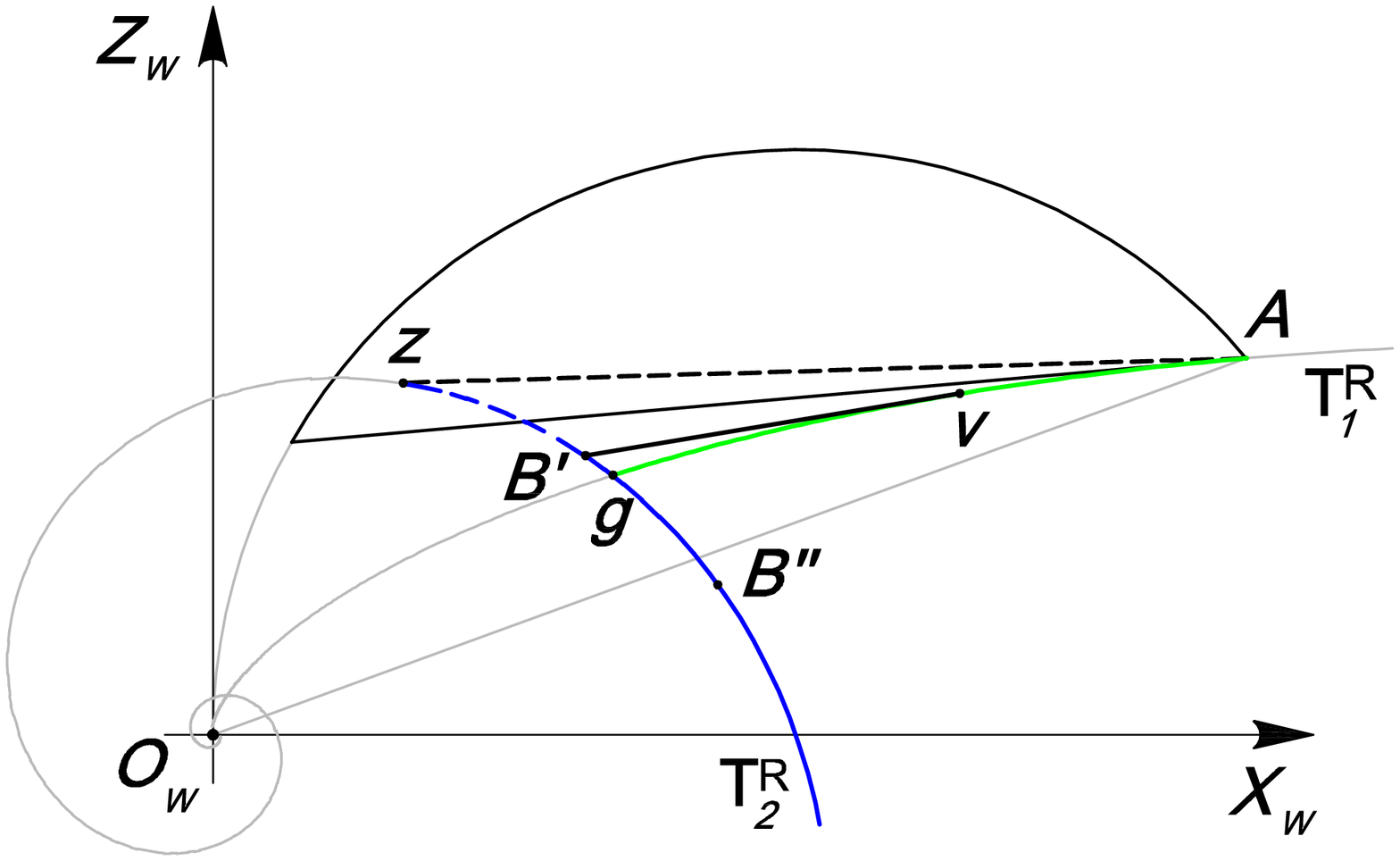}}
\caption{Examples of paths shortened in proposition~\ref{prop:FirstCaseS+} for the Side case.}
\label{fig:Prop3}
\end{figure*}

Proposition~\ref{prop:FirstCaseS+} implies that paths of type $S^-*E^-_1$ and $S^-*E^-_2$ are not optimal. Indeed, they can be shortened by $S^-E^-_1$ and $E^-_2S^-$, respectively (see fig.~\ref{fig:Prop3} for the Side case).

By using all previous results, a sufficient family of optimal paths is
obtained in the following important theorem.
\begin{theorem}
\label{th:CS}
For $\frac{\delta}{2}<\Gamma\leq \frac{\pi}{2}$, i.e. Side and Lateral cases, and for any $Q \in D(P)$ to $P$ there exists a shortest path of type $E^+_1*E^-_2S^-E^-_1$ or of type $E^+_1S^+E^+_2*E^-_1$.
For $0\leq \Gamma\leq\frac{\delta}{2}$, i.e. Frontal case, and for any $Q \in D(P)$ to $P$ there exists a shortest path of type $S^+E^+_1*E^-_2S^-$ or
of type $S^+E^+_2*E^-_1S^-$.
\end{theorem}
\begin{proof}
According to all propositions above several concatenations of extremal have been proved to be non optimal.
Considering extremals as node and, possibly optimal, concatenations of extremal as edges of a graph, the sufficient optimal languages $\mathcal{L}_O$ from $Q$ in $D(P)$, for different values of $\Gamma$ and $\delta$, are described in fig.~\ref{fig:OptimalGraph}.
Indeed, it is straightforward to observe that the number of switches between extremals is finite and less or equal to 3, for any value of $\Gamma$ and $\delta$. Hence, the thesis.
\end{proof}
\begin{figure}[t]
\centering
  \begin{tabular}[c]{c}
  \begin{tabular}[c]{cc}
  \includegraphics[width=0.4\columnwidth,angle=90]{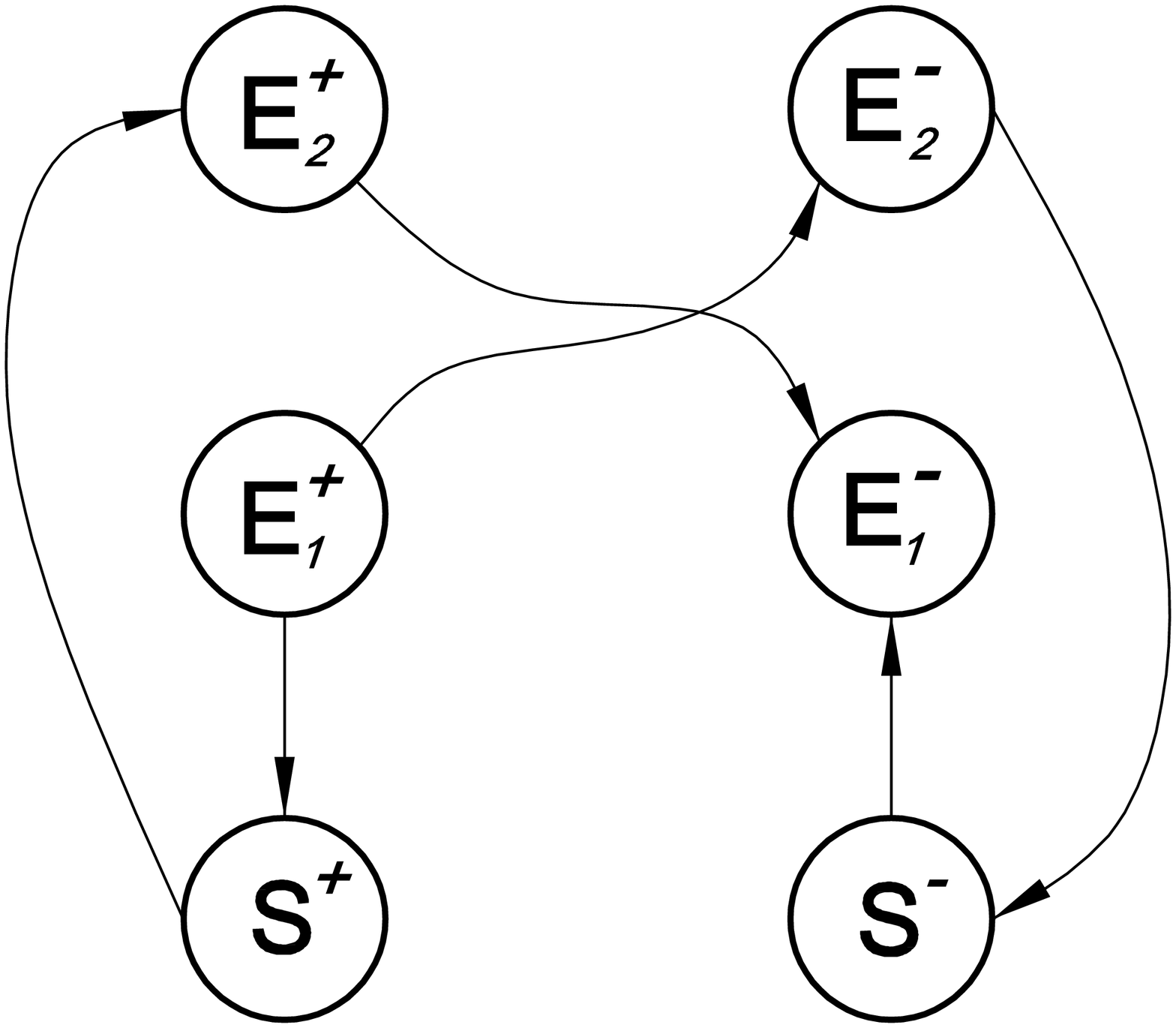} &
  \includegraphics[width=0.4\columnwidth,angle=90]{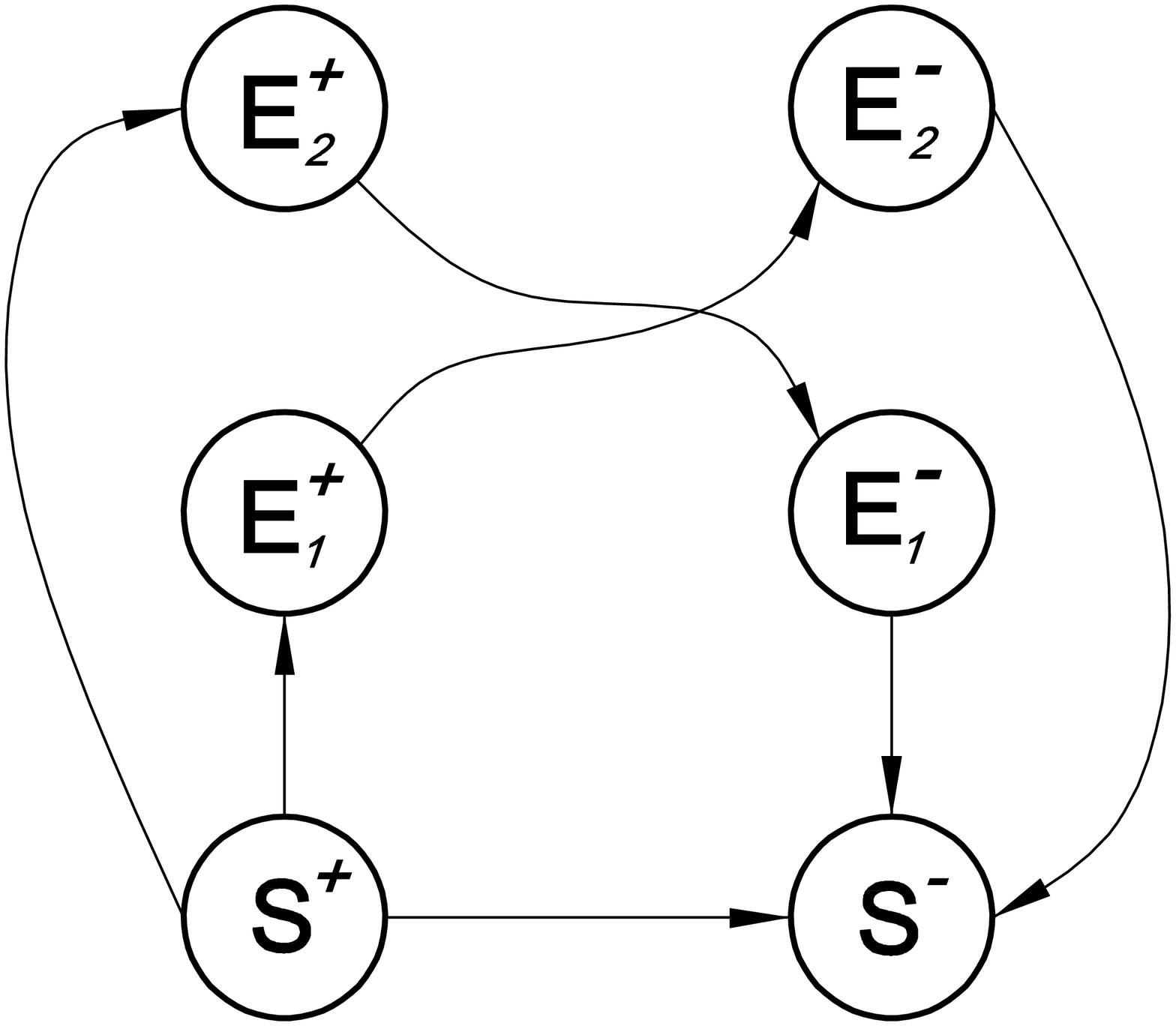} \\
  a) & b)
  \end{tabular}
  \end{tabular}
\caption{Feasible extremals and sequence of extremals from point in D(P): a) in Side and Lateral cases ($\frac{\delta}{2}<\Gamma\leq \frac{\pi}{2}$). b) in Frontal case ($0\leq \Gamma\leq\frac{\delta}{2}$).}
\label{fig:OptimalGraph}
\end{figure}

We now study the length of extremal paths from $C(P)$ to $P$ in the sufficient family above.

Without loss of generality, it is sufficient to study the length of extremal paths of type $E^+_1*E^-_2S^-E^-_1$ only from points $Q$ on the semicircle of $C(P)$ in the upper-half plane (denoted by $CS$). Indeed, up to a rotation, optimal paths of type $E^+_1S^+E^+_2*E^-_1$ from the
rest of $C(P)$ can be easily obtained.
Referring to fig.~\ref{fig:AnalysisOptimalPaths}, let the switching points of the optimal path be denoted by $N$, $M_1$ and $M_2$ or $\bar N$, $\bar M_1$ and $\bar M_2\equiv P$, respectively, depending on the angular values $\alpha_{M_1}$ or $\alpha_{\bar M_1}$. Moreover, in order to do the analysis, it is useful to parameterize the family by the angular value $\alpha_{\bar M_1}$ of the switching point $\bar M_1$ along the arc $C_2(P)$ between $P$ and $Z$ or the angular value $\alpha_{M_1}$ of the switching point $M_1$ along the extremal $E_1$ between $P_F$ and $O_W$.

\begin{theorem}
\label{theo:length}
For any point $Q\in CS$, the length of a path $\gamma\in \mathcal{P}_Q$ of type $E^+_1*E^-_2S^-E^-_1$ is:
\begin{itemize}
\item for $0\leq\alpha_{\bar M_1}\leq\phi_2-\phi_1$, i.e. from $P$ to $Z$ (notice that the last arc has zero length):
\small{
\begin{align}
L & = \rho_P\left\{\frac{\cos\alpha_{M_1}}{\cos\phi_2} + \frac{1}{\cos\phi_1} + \right. \nonumber\\
& -\left.\frac{\cos\phi_1+\cos\phi_2}{\cos\phi_1\cos\phi_2}\,\text{e}^{\left(\psi_Q-\alpha_{M_1}\right)\frac{t_1t_2}{t_2-t_1}}
\left(\frac{\sin\left(\phi_2-\alpha_{M_1}\right)}{\sin\phi_2}\right)^{-\frac{t_1}{t_2-t_1}}\right\}\,,\nonumber\\
&
\end{align}}

\item for $\alpha_{M_1}\geq\phi_2-\phi_1$, i.e. from $Z$ to $O_W$:

\small{
\begin{align}
L & =  \rho_P\left\{\frac{2}{\cos\phi_1} + \text{e}^{-\alpha_{M_1}t_1}\left[\frac{\cos(\phi_2-\phi_1)}{\cos\phi_2}- \frac{1}{\cos\phi_1}+ \right.\right.\nonumber\\
& \left.\left. -\frac{\cos\phi_1+\cos\phi_2}{\cos\phi_1\cos\phi_2}\text{e}^{[\psi_Q-(\phi_2-\phi_1)]\frac{t_1t_2}{t_2-t_1}}
\left(\frac{\sin\phi_1}{\sin\phi_2}\right)^{-\frac{t_1}{t_2-t_1}} \right]\right\}\,,\nonumber\\
&
\end{align}}
\end{itemize}
with $t_1=1/\tan\phi_1$ and $t_2=1/\tan\phi_2$.
\end{theorem}

\begin{figure}[t!]
\centering
\includegraphics[width=0.9\columnwidth]{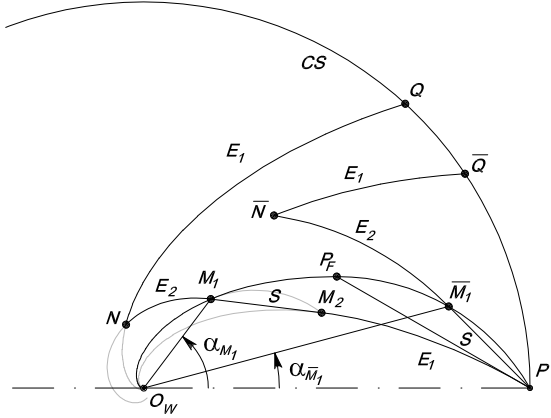}
\caption{Path of type $E^+_1*E^-_2S^-E^-_1$ or the degenerate case of type $E^+_1*E^-_2S^-$  from $Q\in CS$.}
\label{fig:AnalysisOptimalPaths}
\end{figure}

The analytical expression for the length $L$ is based on a direct computation. Having the path's length as a function of two parameters $\alpha_{M_1}$ or $\alpha_{\bar M_1}$ and $\psi_Q$, we are now in a position to minimize the length within the sufficient family.
\begin{theorem}\label{theo:optimalCS}
Given a point $Q\in CS$,
\begin{itemize}
\item for $0\leq\psi_Q\leq\psi_{R_1}:=\frac{\sin(\phi_2-\phi_1)}{\cos\phi_1\cos\phi_2}
    \ln\left(\frac{\cos\phi_1+\cos\phi_2}{\sin\phi_2\sin(\phi_2-\phi_1)}\right)$, optimal
    path is of type $E^+_1*E^-_{2}$;
\item for $\psi_{R_1}\leq\psi_Q\leq\psi_{R_2}$ with $\psi_{R_2}:=(\phi_2-\phi_1)+\psi_{R_1}+\tan\phi_2
    \ln\left(\frac{\sin\phi_1}{\sin\phi_2}\right)$, optimal path is of type $E^+_1*E^-_2S^-$;
\item for $\psi_{R_2}\leq\psi_Q\leq\pi$ the optimal path is $E^+_1*E^-_{1}$ through $O_W$.
\end{itemize}
Moreover, for $\psi_Q=\psi_{R_2}$, any optimal path of type $E^{+}_1*E^{-}_2S^-E^{-}_1$ turns out to have the same
length $\ell$ of optimal path $E^{+}_1*E^{-}_{1}$. Hence, for $\psi_Q=\psi_{R_2}$ also $E^{+}_1*E^{-}_2S^-E^{-}_1$ is optimal.
\end{theorem}

Previous results have been obtained computing first and second derivatives of $L$ and nonlinear minimization techniques.

We are now interested in determining the locus of switching points
between extremals in optimal paths.

\begin{proposition}\label{prop:M_0}
For $Q\in CS$ with $0 < \psi_Q \leq \psi_{R_1} $, the switching locus
is the arc of $E_2$ between $P$
$M=(\rho_P\frac{\sin\phi_2\sin(\phi_2-\phi_1)}{\cos\phi_1+\cos\phi_2},\,\psi_M)$ (included), where $\psi_M=\tan\phi_2\ln\left(\frac{\rho_P}{\rho_M}\right)$.
\end{proposition}

\begin{proof}
From Theorem~\ref{theo:optimalCS}, the optimal path from $Q\in CS$ to $P$ is
of type $E^+_1*E^-_{2}$. For $\psi_Q=\psi_{R_1}$ the intersection between $E^+_1$ and $E^-_{2}$  is $M$.
\end{proof}

\begin{proposition}\label{prop:M_1}
For $Q\in CS$ with $\psi_{R_1} < \psi_Q <\psi_{R_2}$, the loci of
switching points $M_2$ and $N$ are the $\partial SF_2(P)$ and $\partial SF_2(M)$.
\end{proposition}
\begin{proof}
For $Q\in CS$ with $\psi_{R_1} < \psi_Q <\psi_{R_2}$, considering the values of $\alpha_{M_2}$ obtained in the computations of Theorem~\ref{theo:optimalCS} we obtain $M_2\in \partial SF_2(P)$. Furthermore, substituting those values in the equation of the intersection point $N$ between $E_1$ through $Q$ and $E_2$ through $M_2$ we obtain $N\in \partial SF_2(M)$.
\end{proof}
Finally, for $Q\in CS$ with $\psi_{R_2}\leq \psi
<\pi$, the switching locus reduces to the origin $O_W$ since two extremal $E_i$ intersect only in the origin for $i=1,2$.

\section{Shortest paths from any point in the motion plane}
The synthesis on $C(P)$ induce a partition in regions of $D(P)$. Indeed, for any {$Q\in D(P)$}, there exists a point $V\in C(P)$ such that the optimal path $\gamma$ from $V$ to $P$ goes through $Q$. The Bellmann's optimality principle ensure the optimality of the sub--path from $Q$ to $P$. Based on this construction the partition of $C(P)$ is reported in fig.~\ref{fig:SubdivisionInCP}. 

\begin{figure}[t!]
\begin{minipage}[t]{0.4\columnwidth}\centering
\scriptsize{
\begin{tabular}{|c|l|} \hline
Region & Optimal Path \\
\hline \hline
I & $S^-$\\ \hline
II & $E_1^{+}*E_2^-$\\ \hline
II$^\prime$ & $E_2^+*E_1^-$\\ \hline
III & $E_1^+*E_1^-$  \\ \hline
IV & $E_2^-S^-E_1^-$ \\\hline
V & $E_1^+*E_2^-S^-$\\ \hline
V$^\prime$  & $S^+E_2^+*E_1^-$\\ \hline
VI &  $S^-E_1^-$\\ \hline
\end{tabular}}
\end{minipage}
\hspace{-1.2cm}
\begin{minipage}[b]{0.9\columnwidth}
\centering
\includegraphics[width=0.7\columnwidth,angle=90]{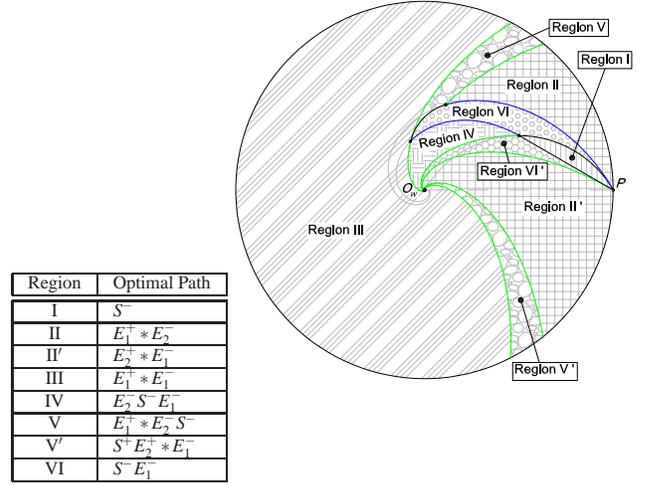}
\end{minipage}
\caption{Optimal synthesis inside $D(P)$.}
\label{fig:SubdivisionInCP}
\end{figure}

For points outside $C(P)$, function $F_Q$ has been defined in~\ref{eq:F} in order to transform paths starting from $Q$ inside $C(P)$ in paths starting from $f_Q(P)=\left(\frac{\rho_P^2}{\rho_Q},-\psi_Q \right)$ outside $C(P)$.

From other properties of $F_Q$, such as Proposition~\ref{prop:length}, we have also that an optimal path is mapped into an optimal path. Hence, the optimal synthesis from points outside $C(P)$ can be easily obtained mapping through map $F_Q$ all borders of regions inside $C(P)$.

\begin{proposition}
\label{prop:borders}
Given a border $\mathbf{B}$ and $Q\in \mathbf{B}$ map $F_Q$ transforms:
\begin{enumerate}
\item $\mathbf{B}=C(P)$ into itself;
\item $\mathbf{B}=\partial SF_2(Q)$ in $\partial SB_1(f_Q(P))$
\item $\mathbf{B}=\partial SF_1(Q)$ in $\partial SB_2(f_Q(P))$
\item $\mathbf{B}=E_i$ in arcs of the same type ($i=1,2$)
\end{enumerate}
\end{proposition}
\begin{proof}
The proof of this proposition can be found in~\cite{SFPB-TRO09}.
\end{proof}
Based on Proposition~\ref{prop:borders}, the optimal synthesis of the entire motion plane is reported in fig.~\ref{fig:CompleteGuercio}.
\begin{figure}[t!]
\centering
\includegraphics[width=0.85\columnwidth,angle=90]{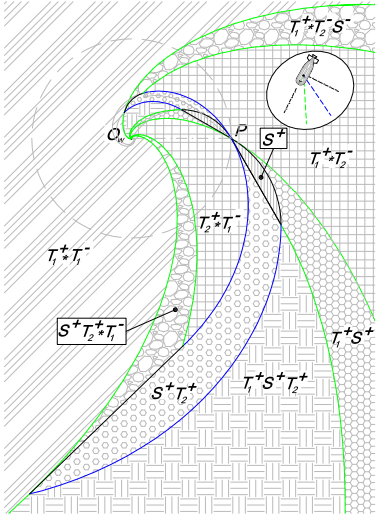}
\caption{Partition of the motion plane for $\frac{\delta}{2}<\Gamma< \frac{\pi-\delta}{2}$ .}
\label{fig:CompleteGuercio}
\end{figure}

\section{Optimal synthesis for generic $\Gamma$}


We first obtain the synthesis of the Borderline Frontal case, i.e. $\Gamma = \frac{\delta}{2}$, reported in fig.~\ref{fig:CompleteForwardDir} from the one obtained in the previous section.
\begin{figure}[t!]
\centering
\includegraphics[width=0.8\columnwidth,angle=90]{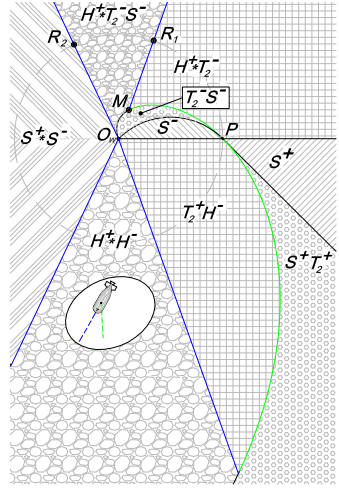}
\caption{Partition of the motion plane for  $\Gamma=\delta/2$ (i.e. a SR border is aligned with the robot motion direction, Borderline Frontal).}
\label{fig:CompleteForwardDir}
\end{figure}

Notice that, $E_1=T_1^R$ of the Side case degenerates in a straight line $H$ through $O_W$ for $\Gamma = \frac{\delta}{2}$. Indeed, referring to fig.~\ref{fig:SubdivisionInCP}, points $M_F$ and $P_F$ degenerate on $O_W$. As a consequence, Region IV, IV and $VI^\prime$ while coordinates $\Psi_{R_1}$ and $\Psi_{R_2}$ of points $R_1$ and $R_2$ can be obtained from values in~\ref{theo:optimalCS} replacing $\phi_1 = 0$.

In the Frontal case, $E_1=H$ becomes a spiral $T_1^L$, straight lines from $P$ and $R_2$ split in straight line and a spiral arc generating the partition reported in fig.~\ref{fig:CompleteAsimm}. In this case, $\phi_1<0$ and points $R_1$ and $R_2$ do not lay on $C(P)$ but on a circle  through $P$ with center $(0,\,-\rho_P\frac{\sin^2\phi_1-\sin^2\phi_2}{2\sin\xi\sin\phi_1\sin\phi_2})$, where $\xi = \frac{t_1+t_2}{t_1\,t_2}\ln\left(\frac{\cos\phi_1+\cos\phi_2}{\sin(\phi_2-\phi_1)}\right)+ \frac{1}{t_1}\ln\left(-\sin\phi_1\right)-\frac{1}{t_2}\ln\left(\sin\phi_2\right)$. Notice that for $\phi_2 = -\phi_1$, this circle  coincide with $C(P)$ and the synthesis proposed in~\cite{SFPB-TRO09} is obtained.

\begin{figure}[t!]
\centering
\includegraphics[width=0.7\columnwidth]{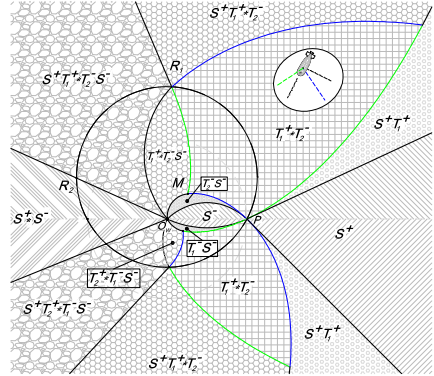}
\caption{Partition of the motion plane for $0\leq \Gamma< \frac{\delta}{2}$, i.e. Frontal case.}
\label{fig:CompleteAsimm}
\end{figure}

Referring again to fig.~\ref{fig:SubdivisionInCP}, in the Borderline Side case ($\Gamma= \frac{\pi-\delta}{2}$, i.e. the SR border is aligned with the axle direction and $\phi_2=\frac{\pi}{2}$), $E_2=T_2^R$ degenerates in $E_2=C$. Points $R_1\equiv M$ and $R_2$ lays on $C(P)$ with $\Psi_{R_1} = \frac{1+\sin\phi_1}{\cos\phi_1}$ and $\Psi_{R_2} = \frac{\pi}{2}-\phi_1+\Psi_{R_1}+\tan\phi_1\ln(\sin\phi_1)$. The obtained synthesis is reported in fig.~\ref{fig:CompleteNonholDir}.
\begin{figure}[t!]
\centering
\includegraphics[width=0.7\columnwidth,angle=90]{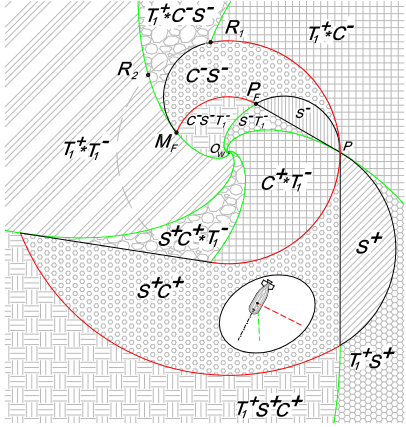}
\caption{Partition of the motion plane for $\Gamma= \frac{\pi-\delta}{2}$ (i.e. a SR border is aligned with the axle direction).}
\label{fig:CompleteNonholDir}
\end{figure}
For the Lateral case $E_2=C$ becomes $E_2=T_2^L$ and the synthesis of the Lateral case, reported in fig.~\ref{fig:CompleteNonholDirAsim}, can be obtained from the one in fig.~\ref{fig:CompleteNonholDir}.
\begin{figure}[t!]
\centering
\includegraphics[width=0.6\columnwidth]{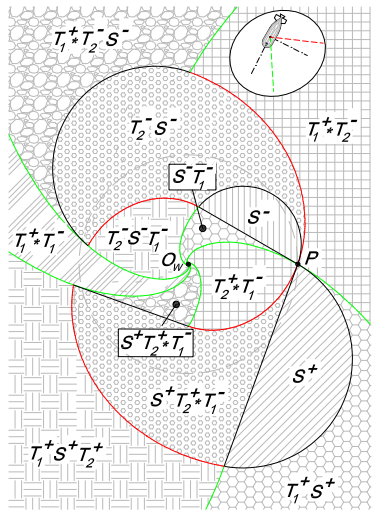}
\caption{Partition of the motion plane for $\frac{\pi-\delta}{2}\leq\Gamma<\frac{\pi}{2}$ (i.e. axle direction is included inside the SR).}
\label{fig:CompleteNonholDirAsim}
\end{figure}

The subdivision of the motion plane in case of $\frac{\pi}{2}<\Gamma\leq\pi$ can be easy obtained by using that one for $0\leq\Gamma\leq\frac{\pi}{2}$ considering optimal path followed in reverse order, i.e. forward arc in backward arc and viceversa. Finally, a symmetry w.r.t. $X_W$ axis of each subdivision of the motion plane for each $\Gamma\in[0,\,\pi]$ allows to obtain the corresponding subdivision for $\Gamma\in[-\pi,\,0]$.

\section{Conclusions and future work}

A complete characterization of shortest paths for unicycle
nonholonomic mobile robots equipped with a limited range side sensor systems has been proposed. A finite sufficient family of optimal paths has been determined based on geometrical properties of the considered problem. Finally, a complete shortest path synthesis to reach a point keeping a feature in sight has been provided.
A possible extension of this work is to consider a bounded 3D SR pointing to any direction with respect to the direction of motion.
A more challenging extension would be considering a different minimization problem such as the minimum time.

\bibliographystyle{plain}
\bibliography{visionbiblio,Fontanelli,cite,OptimalControl}

\appendix
\subsection{Proof of Theorem~\ref{theo:InNOTOut}}
\label{proof:Th1}
\setcounter{theorem}{0}

\begin{theorem}
Given two points $A=(\rho_A,\,\psi_A)$ and $B=(\rho_B,\,\psi_B)$, with $\psi_A>\psi_B$ and $\rho=\rho_A=\rho_B$, and an extremal path $\gamma$ from $A$ to $B$ such that for each point $G$ of $\gamma$, $\rho_G>\rho$, there exists an extremal path $\tilde\gamma$ from $A$ to $B$ such that for each point $\tilde G$ of $\tilde\gamma$, $\rho_{\tilde G}<\rho$ and $\ell(\tilde\gamma)<\ell(\gamma)$.
\end{theorem}
\begin{proof}
Consider a point $Z=(\rho_Z,\,\psi_Z)$ such that $\rho_{Z}=\max_{G\in \gamma}{\rho_G}>\rho$. Let $\gamma_1$ and $\gamma_2$ the sub--paths of $\gamma$ from $Z$ to $B$ and from $Z$ to $A$.

The sub--path $\gamma_1$, is rotated and scaled (contracted of factor $\frac{\rho}{\rho_Z}< 1$) such that $Z$ is transformed in $A$ obtaining a path $\tilde\gamma_1$ from $A$ to $\tilde Z =(\frac{\rho^2}{\rho_Z},\,\psi_A+\psi_B-\psi_Z)$. Similarly, $\gamma_2$, can be rotated and scaled with the same scale factor but different rotation angle w.r.t. $\gamma_1$ such that $Z$ is transformed in $B$, see fig.~\ref{fig:teorema1}. After geometrical considerations, it is easy to notice that the obtained path $\tilde\gamma_2$ starts in $B$ and ends in $\tilde Z$.

The obtained paths are a contraction of $\gamma_1$ and $\gamma_2$ respectively and hence shorter. Moreover, any point $G$ of $\gamma_1$ or $\gamma_2$ has $\rho_G>\rho$ hence is scaled in $\tilde G$ of $\tilde\gamma_1$ or $\tilde\gamma_2$ with $\rho_{\tilde G} = \frac{\rho\rho_G}{\rho_Z}<\rho$.

Concluding, we have obtained a shorter path from $A$ to $B$ that evolves completely in the disk of radius $\rho$.
\end{proof}

\subsection{Proof of Theorem~\ref{theo:rho_nomax}}
\label{prof:Th2}
\setcounter{theorem}{1}

\begin{theorem}
Any path consisting in a sequence of a backward extremal arc followed by a forward extremal arc is not optimal.
\end{theorem}
\begin{proof}
Observe that the distance from $O_W$ is strictly increasing along backward extremal arcs (i.e. $S^-$, $E^-_1$, $E^-_2$ with $E_2\neq C$) and strictly
decreasing along forward extremal arcs (i.e. $S^+$,  $E^+_1$, $E^+_2$ with $E_2\neq C$).
For continuity of paths, for any sequence of a backward extremal followed by a
forward one, there exist points $A$ and $B$ that verify hypothesis of
Theorem~\ref{theo:InNOTOut}, hence it is not optimal.

Any sequence consisting in an extremal $S$ (or $E_1$) of length $\ell$ and an extremal $E_2=C$ (in any order and direction) is inscribed in two circumferences centered in $O_W$.
Hence, the shortest sequence is the one with $E_2=C$ along the circle  of smaller radius necessarily preceded by a forward $S$ (or $E_1$) of same length $\ell$.

Concluding, in an optimal path a forward arc cannot follow a backward arc.
\end{proof}

\end{document}